\documentclass{article}

\usepackage{natbib}
\usepackage[utf8]{inputenc} 
\usepackage[T1]{fontenc}    
\usepackage{hyperref}       
\usepackage{url}            
\usepackage{booktabs}       
\usepackage{amsfonts}       
\usepackage{nicefrac}       
\usepackage{microtype}      
\usepackage{xcolor}         

\usepackage{hyperref}       
\usepackage{url}            
\usepackage{booktabs}       
\usepackage{amsfonts}       

\usepackage{graphicx}
\usepackage{subcaption}

\usepackage{tcolorbox}
\definecolor{darkred}{RGB}{150,0,0}
\definecolor{darkgreen}{RGB}{0,150,0}
\definecolor{darkblue}{RGB}{0,0,150}
\hypersetup{colorlinks=true, linkcolor=red, citecolor=darkblue, urlcolor=darkblue}

\usepackage{tikz}
\usepackage{amssymb}
\usepackage{amsmath}
\usepackage{amsmath,amsfonts,amssymb,amsthm}
\usepackage{mathtools}
\usepackage{commath}
\usepackage{flexisym}
\usepackage[utf8]{inputenc}
\usepackage{csquotes}
\usepackage[english]{babel}

\usepackage{color}
\usepackage{bbm}

\usepackage{amsmath}
\DeclareMathOperator*{\argmax}{arg\,max}

\newtheorem{lemma}{Lemma}

\newtheorem{theorem}{Theorem}
\newtheorem{myassum}{Assumption}

\theoremstyle{definition}
\newtheorem{definition}{Definition}

\usepackage[shortlabels,inline]{enumitem}
\usepackage{lipsum}

\usepackage{enumitem}
\usepackage{xcolor}
\usepackage[linesnumbered,ruled,vlined]{algorithm2e}

\SetCommentSty{mycommfont}

\SetKwInput{KwInput}{Input}                
\SetKwInput{KwOutput}{Output} 
\usepackage[margin=1in]{geometry}

\newcommand{\tocite}[1]{\textcolor{green}{[CITE]}}

\usepackage{xspace}

\newcommand{\Unif}{{\rm Unif}}
\newcommand{\Group}{{\rm Group}\xspace}
\newcommand{\ExpPh}{{\rm ExpPh}\xspace}
\newcommand{\Planning}{{\rm Planning}\xspace}
\newcommand{\distMT}{DistMT-LSVI\xspace}
\newcommand{\csep}{c_{\rm Sep}}

\newcommand{\la}{\lambda}

\newcommand{\nn}{\nonumber}

\newcommand{\bal}{\begin{align}}
\newcommand{\eal}{\end{align}}

\newcommand{\phib}{\boldsymbol{\phi}}

\newcommand{\qb}{\mathbf{q}}

\newcommand{\etab}{\boldsymbol{\eta}}

\newcommand{\A}{\mathbf{A}}

\newcommand{\Yb}{\mathbf{Y}}

\newcommand{\x}{\mathbf{x}}

\newcommand{\y}{\mathbf{y}}

\newcommand{\Iden}{\mathbf{I}}

\newcommand{\Fc}{\mathcal{F}}

\newcommand{\Sc}{{\mathcal{S}}}

\newcommand{\Vc}{\mathcal{V}}
\newcommand{\Dc}{\mathcal{D}}

\newcommand{\Nc}{\mathcal{N}}

\newcommand{\Cc}{\mathcal{C}}
\newcommand{\Mc}{\mathcal{M}}

\newcommand{\Ac}{\mathcal{A}}

\newcommand{\Oc}{\mathcal{O}}

\newcommand{\beq}{\begin{equation}}
\newcommand{\eeq}{\end{equation}}
\newcommand{\bea}{\begin{align}}
\newcommand{\eea}{\end{align}}

\newcommand{\Otilde}{\tilde\Oc}

\newcommand{\Pb}{\mathbb{P}}

\newcommand{\mub}{\boldsymbol \mu}

\newcommand{\thetab}{\boldsymbol\theta}

\newcommand{\Lambdab}{\boldsymbol\Lambda}

\usepackage{authblk}
\author[1]{Sanae Amani}
\author[2]{Khushbu Pahwa}
\author[3]{Vladimir Braverman}
\author[4]{Lin F. Yang}
\affil[1,2,4]{University of California, Los Angeles}
\affil[3]{Rice University}
{
    \makeatletter
    \renewcommand\AB@affilsepx{, \protect\Affilfont}
    \makeatother
    \affil[1]{samani@ucla.edu}
\affil[2]{khushbu16pahwa@g.ucla.edu}
\affil[3]{vb21@rice.edu}
\affil[4]{linyang@ee.ucla.edu}
}

\title{Scaling Distributed Multi-task Reinforcement Learning\\ with Experience Sharing}

\begin{document}

\sloppy
\date{}
\maketitle

\begin{abstract}

Recently, DARPA launched the ShELL program, which aims to explore how experience sharing can benefit distributed lifelong learning agents in adapting to new challenges. In this paper, we address this issue by conducting both theoretical and empirical research on distributed multi-task reinforcement learning (RL), where a group of $N$ agents collaboratively solves $M$ tasks without prior knowledge of their identities. We approach the problem by formulating it as linearly parameterized contextual Markov decision processes (MDPs), where each task is represented by a context that specifies the transition dynamics and rewards. To tackle this problem, we propose an algorithm called \distMT. First, the agents identify the tasks, and then they exchange information through a central server to derive $\epsilon$-optimal policies for the tasks. Our research demonstrates that to achieve $\epsilon$-optimal policies for all $M$ tasks, a single agent using \distMT needs to run a total number of episodes that is at most $\Otilde({d^3H^6(\epsilon^{-2}+\csep^{-2})}\cdot M/N)$, where $\csep>0$ is a constant representing task separability, $H$ is the horizon of each episode, and $d$ is the feature dimension of the dynamics and rewards. Notably, \distMT improves the sample complexity of non-distributed settings by a factor of $1/N$, as each agent independently learns $\epsilon$-optimal policies for all $M$ tasks using $\Otilde(d^3H^6M\epsilon^{-2})$ episodes. Additionally, we provide numerical experiments conducted on OpenAI Gym Atari environments that validate our theoretical findings.

\end{abstract}

\section{Introduction}

In recent years, there has been a growing interest in the development of distributed learning agents, which refer to multiple agents collaborating and communicating to collectively solve learning or decision-making problems with improved efficiency \citep{mcmahan2017communication,alavi2002comparative}. Concurrently, the field of multi-task learning has emerged, focusing on agents that face multiple tasks and aim to learn policies that optimize performance across all tasks \citep{caruana1997multitask,zhang2018overview,crawshaw2020multi}. The intersection of these two research areas presents scenarios where multiple learning agents cooperate to build multi-purpose embodied intelligence, such as robots operating in weakly structured environments \citep{roy2021machine}. Motivated by these developments, the Defense Advanced Research Projects Agency (DARPA) has launched the Shared-Experience Lifelong Learning (ShELL) program. The program seeks to address how experience sharing can assist distributed lifelong learning agents in effectively adapting to new challenges \citep{darpa:2021}. This research aims to explore the potential benefits of collaborative learning approaches in enhancing the capabilities of distributed agents in dynamic environments.



In this work, we delve into the theoretical and empirical aspects of distributed multi-task reinforcement learning (RL) \citep{li2019accelerating,lazaric2010bayesian}. In this setting, a group of $N$ agents collaboratively tackles $M$ tasks in a pure exploration manner, with the task identifications initially unknown. It is assumed that the tasks exhibit variations in rewards and transition dynamics, but share the same state and action spaces \citep{sodhani2021multi}. In consecutive learning rounds, the agents are assigned tasks drawn from a uniform distribution ($\Unif([M])$). The ultimate objective is for the agents to cooperate effectively, ensuring that by the end of the exploration phase, they all possess $\epsilon$-optimal policies for all tasks, while minimizing the total number of episodes required during the exploration phase to execute a policy.


Formally, we consider an episodic setup based on the framework of contextual MDP \citep{abbasi2014online,hallak2015contextual}. It repeats the following steps:
\begin{enumerate*}[label=\textit{\arabic*)}]
    \item At the beginning of each learning round, each agent receives an unknown context specifying the assigned task.
    \item Each agent initially spends certain number of episodes interacting with the corresponding task's environment, which together with communications with others through the server, help it identify the task and whether it has been solved by any other agents before.
    \item If the agent determines that the task has already been solved, it obtains the necessary statistics from the agent who previously solved the task to determine the task's $\epsilon$-optimal policy. Otherwise, it initiates the learning process for the $\epsilon$-optimal policy from scratch.
\end{enumerate*}

The performance of each agent is evaluated based on the total number of episodes required for it to interact with the unknown environments of assigned tasks and communicate with other agents through the server, ultimately gaining access to $\epsilon$-optimal policies for all $M$ tasks. To maximize the benefits of the collaborative learning process, we aim for this number to scale as $\frac{M}{N\epsilon^2}$, ensuring efficient utilization of the agents' cooperative nature. Remarkably, this scaling results in a multiplicative reduction of exploration episodes by a factor of $1/N$ compared to non-distributed settings.


On the multi-task side, the closest lines of work are \citet{modi2020no,abbasi2014online,hallak2015contextual,modi2018markov,kakade2020information} for contextual MDP and \citet{wu2021accommodating,abels2019dynamic} for the dynamic setting of multi-objective RL, which study the sample complexity for arbitrary task sequences; however, they either assume the problem is tabular with finite state and action spaces or require a model-based planning oracle with unknown complexity. Importantly, none of the existing works properly addresses the need for a distributed learning framework, which creates a large gap between the abstract setup and practice need of speeding up the learning process.

In this paper, we aim to establish a foundation for designing agents meeting these practically important requirements. As the first step, here, we study distributed multi-task RL with linear representation. We suppose that the contextual MDP is linearly parameterized~\citep{yang2019sample, jin2020provably} and agents need to learn a multi-task policy based on this linear representation. However, the fact that tasks' identities are unknown to the agents make the efficient communications challenging. To overcome this hurdle, tasks must possess distinguishable features that enable agents to identify them and effectively share knowledge by communicating specific measurements related to the environments/tasks they interact with. In particular, we introduce a task-separability assumption, which is sufficient to ensure that the agents are able to distinguish between their assigned tasks throughout the learning process so that they can share information only when needed and when a task has already been solved. Under these assumptions, we propose a provably efficient distributed multi-task RL algorithm, Distributed Multi-Task Least Value Iteration (\distMT). We show that in order to obtain $\epsilon$-optimal policies for all $M$ tasks, the total number of episodes a single agent needs to run \distMT is at most $\Otilde({d^3H^6(\epsilon^{-2}+\csep^{-2})}\cdot M/N)$, where $\csep>0$ is a constant characterizing task separability, $H$ is horizon of each episode and $d$ is the feature dimension of the dynamics and rewards. Remarkably, \distMT improves the sample complexity of non-distributed settings, where each agent separately learns all $M$ tasks' $\epsilon$-optimal policies using $\Otilde(d^3H^6M\epsilon^{-2})$ episodes, by a factor of $1/N$.

Finally, we present numerical experiments on OpenAI Gym Atari environments that corroborate our theoretical findings.

\textbf{Notation.} 
Throughout, we use lower-case letters for scalars, lower-case bold letters for vectors, and upper-case bold letters for matrices. The Euclidean norm of $\x$ is denoted by $\norm{\x}_2$. We denote the transpose of any column vector $\x$ by $\x^\top$. For any vectors $\x$ and $\y$, we use $\langle \x,\y\rangle$ to denote their inner product. Let $\A$ be a positive semi-definite $d\times d$ matrix and $\boldsymbol \nu \in\mathbb R^d$. The weighted 2-norm of $\boldsymbol \nu$ with respect to $\A$ is defined by $\norm{\boldsymbol \nu}_\A = \sqrt{\boldsymbol \nu^\top \A \boldsymbol \nu}$. For a positive integer $n$, $[n]$ denotes the set $\{1,2,\ldots,n\}$. Finally, we use standard $\Otilde$ notation for big-O notation that ignores logarithmic factors.

\section{Problem formulation} \label{sec:formulate}
\paragraph{Finite-horizon contextual MDP.}
We consider the problem of learning $M$ tasks, each of which is modeled by an MDP. Each task $m\in[M]$ is associated with an MDP $\Mc_m=(\Sc,\Ac, H,\Pb_m, r_m)$, where the state space $\Sc$, the action space $\Ac$, and the horizon $H$ (length of each episode) are shared amongst all tasks, and $\Pb_m=\{\Pb_{m,h}\}_{h=1}^H$ are the unknown transition probabilities, and $r=\{r_{m,h}\}_{h=1}^H$ are the unknown reward functions specific to task $m$. For $(m,h)\in[M]\times[H]$, $r_{m,h}(s,a)$
denotes the reward function of task $m$ at step $h$, whose range is assumed to be in $[0,1]$, and $\Pb_{m,h}(s^\prime|s,a)$ denotes the probability of transitioning to state $s^\prime$ upon playing action $a$ at state $s$ while solving task $m$ at step $h$. To simplify the notation, for any function $f$, we write $\mathbb{P}_{m,h}[f](s,a)\coloneqq\mathbb{E}_{s^\prime\sim\mathbb{P}_{m,h}(.|s,a)}[f(s^\prime)]$.

\paragraph{Policy and value functions.}
A policy $\pi=\{\pi_h\}_{h=1}^H$ is a sequence where $\pi_h:\Sc\to\Ac$ determines the agent's action at step $h$. Given $\pi$, we define state value function of task $m$ as
$ V_{m,h}^\pi(s)\coloneqq\mathbb{E}[\sum_{h^\prime=h}^H r_{m,h^\prime}\left(s_{h^\prime},\pi_{h^\prime}(s_{h^\prime}))\vert s_h=s\right]$, where the expectation is with respect to policy $\pi$ and transition probability $\mathbb{P}_m$. We also define its action-value function as $Q_{m,h}^\pi(s,a) \coloneqq r_{m,h}(s,a) + \mathbb{P}_{m,h}[V^\pi_{m,h+1}(.)](s,a)$, where $Q_{m,H+1}^\pi=0$. We denote the optimal policy of task $m$ as $\pi_{m,h}^\ast(s) \coloneqq \sup_{\pi} V_{m,h}^\pi(s)$, and let $V_{m,h}^\ast \coloneqq V_{m,h}^{\pi^\ast}$ and $Q_{m,h}^\ast \coloneqq Q_{m,h}^{\pi^\ast}$ denote the optimal value functions associated with task $m$. Lastly, we recall the Bellman equation of the optimal policy: 
\begin{align}
Q_{m,h}^\ast(s,a)=r_{m,h}(s,a)+\mathbb{P}_{m,h}[V^\ast_{m,h+1}(.)](s,a),\quad
V_{m,h}^\ast(s)=\max_{a\in\Ac}Q_{m,h}^\ast(s,a)\label{eq:bellmanforoptimal}.
\end{align}

\paragraph{Interaction protocol.} We consider a network of $N$ agents acting cooperatively to efficiently solve the above-stated $M$ tasks. The learning framework consists of multiple learning \emph{rounds}, each of which consists of multiple \emph{episodes}. At each round $t$, each agent $i\in [N]$ is given a task $m_{i,t}\sim\Unif([M])$, whose ID is \emph{not} known to the agent. Agent $i$ interacts with task $m_{i,t}$ in a number of episodes and collects trajectory $s_{1,i}^{k,t}, a_{1,i}^{k,t}, r_{1,i}^{k,t}, s_{2,i}^{k,t}, a_{2,i}^{k,t}, r_{2,i}^{k,t}, \ldots, s_{H,i}^{k,t}, a_{H,i}^{k,t}, r_{H,i}^{k,t}$ at episode $k$, where the initial state $s_{1,i}^{k,t}$ is a fixed initial state $s_0$. The agents are also allowed to communicate with each other via a central server. Both policies and the communicated information of each agent may only depend on previously observed rewards and communication received from other agents.

\paragraph{Linear Function Approximation.}
We focus on MDPs
with linear transition kernels and reward functions \citep{jin2020provably,wang2020reward, yang2019sample}
that are encapsulated in the following assumption.

\begin{definition}\label{def:linearMDP}
$\Mc=(\Sc,\Ac, H,\Pb, r)$ is a linear MDP with feature map $\phib:\Sc\times\Ac\rightarrow \mathbb{R}^d$ if for any $h\in[H]$, 
there exist a vector $\etab_{h}$ and $d$ measures $\mub_{h}\coloneqq[{\mu_{h}}^{(1)},\ldots,{\mu_{h}}^{(d)}]^\top$ over $\Sc$ such that 
 $\Pb_{h}(.|s,a)=\left\langle \mub_{h}(.), \phib(s,a)\right\rangle$ and  $r_{h}(s,a)=\left\langle\etab_{h},\phib(s,a)\right\rangle$, for all $(s,a)\in\Sc\times\Ac$. Without loss of generality, $\norm{\phib(s,a)}_2\leq 1$, $\norm{\mub_{h}(s)}_2\leq \sqrt{d}$, and $\norm{\etab_{h}}_2\leq \sqrt{d}$ for all $(s,a,h)\in\Sc\times\Ac\times[H]$. 
\end{definition}

\begin{myassum}[Linear MDPs]\label{assum:linearMDP}
For each task $m\in[M]$, $\Mc_m=(\Sc,\Ac, H,\Pb_m, r_m)$ is a linear MDP with feature map $\phib:\Sc\times\Ac\rightarrow \mathbb{R}^d$.
\end{myassum}

Here, we introduce a task-separability assumption that allows agents to distinguish and identify different tasks they are assigned throughout the learning process. 

\begin{myassum}[Task separability]\label{assum:separ}
There exists a known positive constant $\csep>0$, such that for any task pair $(m,m^\prime)\in[M]\times[M], m\neq m^\prime$ it holds that $\abs{V^\ast_{m,1}(s_0) - V^\ast_{m^\prime,1}(s_0)}>\csep$.
\end{myassum}

\paragraph{Goal.}
We say a policy $\pi$ is $\epsilon$-optimal with respect to task $m$ if $V_{m,1}^\pi(s_0)\geq V_{m,1}^\ast(s_0)-\epsilon$. The goal is to design a cooperative and exploratory algorithm that will be run by all the agents in parallel and when it stops, all the agents have access to $\epsilon$-optimal policies for all tasks while it minimizes the total number of episodes during the exploration phase at which a policy must be executed.


\section{Distributed Multi-Task LSVI}\label{sec:alg}
In this section, we give a description of our algorithm, Distributed Multi-Task Least Value Iteration (\distMT), presented in Algorithm \ref{alg:main} for a single agent $i$, and is run in parallel by all $N$ agents. This algorithm is inspired by the exploration phase for the reward-free setting in \citet{wang2020reward}. At the beginning of each round, after task allocation, every agent first needs to determine if its assigned task has been already solved or not. However, since the task labels are unknown, communicating task labels is not an option to figure out whether a task has been solved or not. Therefore, agents need to communicate another measure that under Assumption \ref{assum:separ} will potentially help agents distinguish and identify tasks. In view of this, when tasks are allocated at the beginning of learning round $t$, each agent $i$ first needs to spend $K_1$ episodes to calculate certain statistics of its corresponding task's unknown parameters (Lines \ref{line:K1explore} and \ref{line:K1plan}). These statistics then will be shared with the server, who relying on Assumption \ref{assum:separ} and checking a certain separability measure (Line \ref{line:cond}), informs agent $i$ whether its assigned task has been solved before, and if so, lets it know of its task's label. In such cases, agent $i$ will screen the look-up table $F$, which is continuously being updated and shared with all the agents by the server to find its task's label and its corresponding $\epsilon$-optimal policy. If the server determines that a task assigned to agent $i$ has not been solved before, it lets agent $i$ know that it needs to spend another $K_2$ episodes of exploration to learn the $\epsilon$-optimal policy of this new task (Lines \ref{line:K2explore} and \ref{line:K2plan}). Agent $i$ then share the necessary statistics to calculate this policy with the server. At the end of round $t$, when the server has gathered the necessary information from all the agents, it updates two look-up tables $G$, which includes statistics determining task-separability and $F$, which includes task labels and statistics determining their corresponding $\epsilon$-optimal policies (Line \ref{line:lookup tables}). Finally, agent $i$ will receive all the changes made to $F$ by the server.


\begin{algorithm}[ht]
\caption{\distMT($\delta,\epsilon$) for agent $i$}\label{alg:main}
\DontPrintSemicolon
   {\bf Set:} $\beta_1 = c_{\beta_1} Hd\sqrt{\log( dHM\delta^{-1}\csep^{-1})}$ for some $c_{\beta_1}>0$, $K_1 = c_{K_1}d^3H^6\log(dHM\delta^{-1}\csep^{-1})/\csep^2$ for some $c_{K_1}>0$, $\beta_2 = c_{\beta_2} Hd\sqrt{\log( dHM\delta^{-1}\epsilon^{-1})}$ for some $c_{\beta_1}>0$, $K_2 = c_{K_2}d^3H^6\log(dHM\delta^{-1}\csep^{-1})/\epsilon^2$ for some $c_{K_2}>0$, $T = \frac{6M\log(M/\delta)}{N}, \ell=0$, $G_m = F_m = \emptyset,~\forall m\in[M]$\;
   \For{{\rm rounds} $t=1,\ldots, T$}{
    $\Dc_i^{1,t} = \ExpPh(\beta_1,K_1)$ \quad \quad \textcolor{blue}{with unknown parameters $\mathbb{P}_{m_{i,t}}$ and $r_{m_{i,t}}$}\label{line:K1explore}\;
    $\left(\{(\thetab_{h,i}^{1,t}, \Lambdab_{h,i}^{1,t})\}_{h\in[H]}, V_{1,i}^{1,t}(s_0) \right)= \Planning(\Dc_i^{1,t})$\label{line:K1plan}\;
    Send $\left(\{(\thetab_{h,i}^{1,t}, \Lambdab_{h,i}^{1,t})\}_{h\in[H]}, V_{1,i}^{1,t}(s_0) \right)$ to the server\;
    \If{$\exists m\in[\ell]$ such that for all $V\in G_m$, $\abs{V- V_{1,i}^{1,t}(s_0)}\leq \csep/2$\label{line:cond} }{
    Server informs agent $i$ of the task label $m_{i,t} = m$ \quad \quad \textcolor{blue}{Agent has already access to the $\epsilon$-optimal policy for task $m_{i,t}$}\;
    Let $\{(\thetab_h, \Lambda_h)\}_{h\in[H]} = F_m$\;
    Return policy $\pi_i^t = \{\pi_{h,i}^t\}_{h=1}^H$, where $\pi_{h,i}^t(.) = \argmax_{a\in\Ac}Q_h(.,a)$, $Q_h(.,.) =  \min\left\{\langle\thetab_h,\phib(.,.)\rangle+u_h(.,.),H\right\}$,  $u_h(.,.) = \min\left\{\beta_2\norm{\phib(.,.)}_{\Lambdab_h^{-1}},H\right\}$.
} 
\Else
{

$\Dc_i^{2,t} = \ExpPh(\beta_2,K_2)$ \quad \quad \textcolor{blue}{with unknown parameters $\mathbb{P}_{m_{i,t}}$ and $r_{m_{i,t}}$}\label{line:K2explore}\;
$\left(\{(\thetab_{h,i}^{2,t}, \Lambdab_{h,i}^{2,t})\}_{h\in[H]}, V_{1,i}^{2,t}(s_0)\right)= \Planning(\Dc_i^{2,t})$\label{line:K2plan}\;
Send $\{(\thetab_{h,i}^{2,t}, \Lambdab_{h,i}^{2,t})\}_{h\in[H]}$ to the server\;
 Return policy $\pi_i^t = \{\pi_{h,i}^t\}_{h=1}^H$, where $\pi_{h,i}^t(.) = \argmax_{a\in\Ac}Q_h(.,a)$, $Q_h(.,.) =  \min\left\{\langle\thetab_{h,i}^{2,t},\phib(.,.)\rangle+u_h(.,.),H\right\}$,  $u_h(.,.) = \min\left\{\beta_2\norm{\phib(.,.)}_{\left(\Lambdab_{h,i}^{2,t}\right)^{-1}},H\right\}$
}
Server updates $\{G_m\}_{m=1}^M, \{F_m\}_{m=1}^M, f = \Group(\{G_m\}_{m=1}^M, \{F_m\}_{m=1}^M, \ell, t)$\label{line:lookup tables}\;
Receive $\{F_m\}_{m=\ell+1}^{\ell+f}$ from the server.\;
$\ell = \ell+f$
}

\end{algorithm}

\begin{algorithm}[ht]
\caption{\ExpPh($\beta$, $K$)}\label{alg:explore}
\DontPrintSemicolon
  {\bf Unknown parameters:} $\mathbb{P}, r$\;
   {\bf Set:} $Q_{H+1}^k(.,.,.)=0,~\forall k\in[K]$\;
   \For{ $k=1,\ldots, K$}{
   Observe the initial state $s_1^k = s_0$\;
   \For{$h=H,H-1,\ldots,1$}{
   $\Lambdab_h^k = \Iden_{d}+\sum_{\tau=1}^{k-1}\phib_h^\tau{\phib_h^\tau}^\top$\;
   $u_h^k(.,.) = \min\left\{\beta\norm{\phib(.,.)}_{\left(\Lambdab_h^k\right)^{-1}},H\right\}$\;
   Define the exploration-driven reward function $\tilde u_h^k(.,.) = u_h^k(.,.)/H$\;
   $\thetab_h^k=(\Lambdab_h^k)^{-1}\sum_{\tau=1}^{k-1}\phib_h^\tau[
    V_{h+1}^k(s_{h+1}^\tau)]$\;
    $Q_h^k(.,.) =  \min\left\{\langle\thetab_h^k,\phib(.,.)\rangle+u_h^k(.,.)+\tilde u_h^k(.,.),H\right\}$ and $V_h^k(.)=\max_{a\in\Ac}Q_h^k(.,a)$\;
    $\pi_h^k(.) = \argmax_{a\in\Ac} Q_h^k(.,a)$
   }

   \For{$h=1,2,\ldots,H$}{
   Take action $a_h^k = \pi_h^k(s_h^k)$, and observe $s_{h+1}^k\sim \mathbb P_h^k(.|s_h^k,a_h^k)$ and $r_h^k = r_h(s_h^k,a_h^k)$
   }
   
}
{\bf Return:} $\Dc = \{(s_h^k,a_h^k,r_h^k)\}_{(h,k)\in[H]\times[K]}$
\end{algorithm}

\begin{algorithm}[ht]
\caption{\Planning($\{(s_h^k,a_h^k,r_h^k)\}_{(h,k)\in[H]\times[K]}$)}\label{alg:planning}
\DontPrintSemicolon
   \For{$h=H,H-1,\ldots,1$}{
   $\Lambdab_h \coloneqq \Iden_{d}+\sum_{\tau=1}^{K}\phib_h^\tau{\phib_h^\tau}^\top$\;
   $u_h(.,.) = \min\left\{\beta\norm{\phib(.,.)}_{\Lambdab_h^{-1}},H\right\}$\;
   $\thetab_h=(\Lambdab_h)^{-1}\sum_{\tau=1}^{K}\phib_h^\tau[r_h^\tau+
   V_{h+1}(s_{h+1}^\tau)]$\;
    $Q_h(.,.) =  \min\left\{\langle\thetab_h,\phib(.,.)\rangle+u_h(.,.),H\right\}$ and $V_h(.)=\max_{a\in\Ac}Q_h(.,a)$
   }

{\bf Return:} $\left(\{(\thetab_h, \Lambdab_h)\}_{h\in[H]}, V_1(s_0)\right)$

\end{algorithm}

\begin{algorithm}[ht]
\caption{\Group for the server($\{G_m\}_{m=1}^M, \{F_m\}_{m=1}^M, \ell, t$)}
\DontPrintSemicolon
   {\bf Initialization:}  $f = 0$\;
   \For{$i=1,\ldots N$}{

   \If{$\exists m\in[\ell]$ such that for all $V\in G_m$, $\abs{V- V_{1,i}^{1,t}(s_0)}\leq \csep/2$}{
   Add $V_{1,i}^{1,t}(s_0)$ to $G_m$.
   }
   \Else
   {
$\ell = \ell+1$, $f = f+1$\;
   Add $V_{1,i}^{1,t}(s_0)$ to $G_\ell$ and add $\left(\{(\thetab_{h,i}^{2,t}, \Lambdab_{h,i}^{2,t})\}_{h\in[H]},\ell\right)$ to $F_\ell$.
   }
   }
   Return $\{G_m\}_{m=1}^M, \{F_m\}_{m=1}^M, f$
\end{algorithm}

\begin{theorem}\label{thm:main}
Let Assumptions \ref{assum:linearMDP} and \ref{assum:separ} hold. Let $\delta\in(0,1)$, $\epsilon\in(0,1)$, $T = 6M\log(M/\delta)/N$, $K_1 = c_{K_1}d^3H^6\log(dHM\delta^{-1}\csep^{-1})/\csep^2$ for sufficiently large $c_{K_1}>0$, and $K_2 = c_{K_2}d^3H^6\log(dHM\delta^{-1}\csep^{-1})/\epsilon^2$ for sufficiently large  $c_{K_2}>0$. Then, if all the agents run Algorithm \ref{alg:main} in parallel, with probability at least $1-\delta$, every agent $i\in[N]$ at every round $t\in[T]$ returns an $\epsilon$-optimal policy $\pi^t_i$ of a given task with unknown label $m_{i,t}$ and at the end of $T$ rounds, every agent $i\in[N]$ has access to all the tasks' $\epsilon$-optimal policies using a total number of at most $T(K_1+K_2)$ episodes.
\end{theorem}
\subsection{Proof Sketch of Theorem \ref{thm:main}}
In this section, we give a proof sketch for Theorem \ref{thm:main}. We start by introducing the following lemmas, which are the foundation of our analysis, and whose complete proofs are given in Appendix \ref{appx:proofofthm}.
The following lemma shows that the estimated value functions are optimistic with high probability and close to the optimal value functions. It confirms that the condition in Line \ref{line:cond} of Algorithm \ref{alg:main} enable agents to identify the tasks and prevents them from solving a task that has already been solved, and an agent can directly inquire about an $\epsilon$-optimal policy of such a task from the server. 

\begin{lemma} \label{lemm:confidence1}
Under the setting of Theorem \ref{thm:main}, for all $(i,t)\in[N]\times[T]$, with probability at least $1-\delta$, it holds that
\begin{align}
    0\leq V_{1,i}^{1,t}(s_0)-V_{m_{i,t},1}^\ast(s_0)\leq \csep/8.
\end{align}
\end{lemma}

Using Lemma \ref{lemm:confidence1}, in the following lemma, we design the condition in Line \ref{line:cond}.
\begin{lemma}\label{lemm:separ}
Let $K_1$ be chosen as in Lemma \ref{lemm:confidence1}. Then, conditioned on the event introduced in Lemma \ref{lemm:confidence1}, if $\abs{V_{1,i}^{1,t}(s_0)-V_{1,j}^{1,t^\prime}(s_0)}>\csep/2$, then $m_{i,t}$ and $m_{j,t^\prime}$ are two different tasks; otherwise, they are the same tasks.
\end{lemma}

\begin{lemma} \label{lemm:confidence2}
Under the setting of Theorem \ref{thm:main}, and conditioned on the event defined in Lemma \ref{lemm:separ}, for all $(i,t)\in[N]\times[T]$, with probability at least $1-\delta$, it holds that
\begin{align}
    0\leq V_{m_{i,t},1}^\ast(s_0) - V_{m_{i,t},1}^{\pi_i^t}(s_0)\leq \epsilon.
\end{align}
\end{lemma}

\begin{lemma}\label{lemm:T}
    Under the setting of Theorem \ref{thm:main}, with probability at least $1-\delta$, after $T$ rounds all the tasks have been solved by at least one agent.
\end{lemma}

Conditioned on the events introduced in Lemmas \ref{lemm:separ}, \ref{lemm:confidence2}, and \ref{lemm:T}, only $T = 6M\log(M/\delta)/N$ rounds of learning are sufficient such that all the agents have access to the $\epsilon$-optimal policies of all tasks $m\in[M]$, which proves the second part of the Theorem \ref{thm:main}'s statement.

\section{Experiments}

The experiments use the OpenAI Gym Atari environments for training and evaluation. Our task setting consists of $N=20$ agents and $M=10$ tasks. Following our proposed algorithm, \distMT, the agents share information among each other through a central server which we also refer to as the hub.
As discussed earlier, the objective is to design a cooperative algorithm so that when the algorithm stops, each agent has access to the $\epsilon$-optimal policies for all the tasks.

In the learning stage, each agent learns a sequence of 10 randomly assigned tasks. The number of learning rounds ($T$) equals 10. Each agent runs Algorithm \ref{alg:main} in parallel. At the beginning of each round, each agent is assigned a task. The agent first needs to check if the assigned task has already been completed by other agents or not. Under Assumption \ref{assum:separ}, we achieve this by letting each agent measure task similarity by using a small neural network, which we refer to as SimNet. SimNet is a small Deep Q-Network (DQN) that consists of only two linear layers with one ReLU in between. Before an agent learns a new task, it always initializes SimNet from the same set of parameters. This allows the agent to be trained to rapidly capture patterns of the new task for 10,000 frames within one minute. Then, we use the updated parameters to compare similarity between parameters of other SimNets shared by neighboring agents on their own past tasks with distance metrics,such as Euclidean distance of parameter matrices. In this fashion, we find out the most similar tasks to our current tasks. The agent then queries features such as replay buffers of those learned tasks to train a Lifelong Learning model.

The Lifelong Learning property of the algorithm is implemented by combining an Elastic Weight Consolidation (EWC) module, which is a continual learning algorithm that effectively slows down catastrophic forgetting of past tasks during the convergence of an ongoing task by selectively retaining elasticity of important parameters, and a DQN algorithm. In particular, in the beginning of the training of a new task, the agent first obtains the experience dataset and initialization parameters for a similar task from its own experience or from other agents, then the agent continues to train its model by interacting with the environment. The existing experience will greatly reduce the training time of the new task. For each new task the agent receives, the agents are able to obtain a concise feature map of the task, and use it to query the central server, i.e., obtain the experience from the neighboring agents who have already seen the task. Each agent queries the central server that hosts the neighbors' SimNet parameters to compare the task similarity with its own SimNet, deciding which relevant task information to acquire from the server. Having obtained a list of “similar” agents to communicate with, the agent decompresses learning representations of those tasks for continual learning its model. If there is no similar task available, the LL agent trains the new task from scratch. Otherwise, it leverages learnable information from neighbors to perform LL. The LL agent frees up memory consumed by shared information. In this way, our proposed approach prevents the agents from solving a task that has already been solved, and an agent can directly inquire about an $\epsilon$-optimal policy of a task that has already been solved by other agents from the hub.

Before learning each task, every agent encounters one of the following scenarios: 1) If neither the agent nor its neighbors has encountered a new task before, the agent needs to interact with the Atari server for at least 10000 frames for learning samples to train the Lifelong Learning model. This is equivalent to single-agent training where the agent needs to train each task from scratch. 2) If the task has been learned by its peers, our agent requests the experience replay buffers (ERB). It leverages learnable information, e.g. by performing memory replay on borrowed ERB locally instead of interacting with the Atari-server, saving large amounts of time in learning from more data in a short frame of time.


We implement our RL algorithm with a deep Q-learning (DQN) framework. In this framework, a deep neural network stores the value of a Q-function, which represents the value of a possible move given an observation. The Q-function takes input a game state (e.g., an image, or an internal memory of the Atari game) and an action (e.g., a button on the game sticker), and outputs a value of the action.

To train an RL agent, we need an initialization parameter for the Q-function and an experience dataset of a given task. The dataset consists of tuples of 1) observation embeddings, 2) action, 3) reward, 4) termination status of the current state, and 5) observation embeddings of the next state e.g. RoadRunner imported from OpenAI Gym library.

The experiments include comparison of the following two baselines:
(i) Isolated agent: In this baseline, a single agent performs the conventional lifelong RL to sequentially learn the Atari environments, i.e. without sharing of experience replay buffers. 
(ii) \distMT Agent: In addition to performing lifelong RL on each agent, these agents also share experiences across each other, and hence learning the optimal policies faster.

In order to evaluate our experiment, we use two metrics, defined as Agent Time (AT) and
Agent Reward (AR). Assume there are $N$ agents that collaboratively learn $M$ distinct tasks. Each agent learns a unique sequence of these M tasks. More formally, AT and AR are defined as follows: \begin{enumerate*}[label=\textit{\arabic*)}]
  \item Agent Time (AT) refers to the average time of each of the $N$ agents to achieve 90\% of average normalized scores of all tasks over a permutation of sequences of $M$ distinct tasks. This metric is proportional to the notion of sample complexity for which we provide an upper bound in Theorem \ref{thm:main}.
  \item Agent Reward (AR) refers to the average reward of $N$ agents after being trained for a fixed amount of time on each of the $M$ distinct sequential tasks.
\end{enumerate*}

We have conducted the following experiments to evaluate the performance improvement of the \distMT agent over the Isolated Lifelong Learning agent. The goal of this experiment is to demonstrate the increase of Agent Reward when the number of agents is increasing. We use different numbers of agents that interpolate between 1 and 20. The case of one agent corresponds to the Isolated agent case.

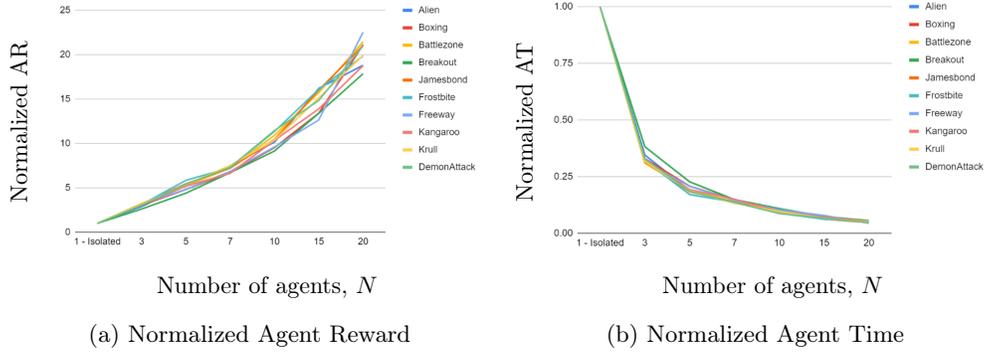
\begin{figure*}[ht]
\centering
\begin{subfigure}{2.6in}
\begin{tikzpicture}
\node at (0,0) {\includegraphics[scale=0.38]{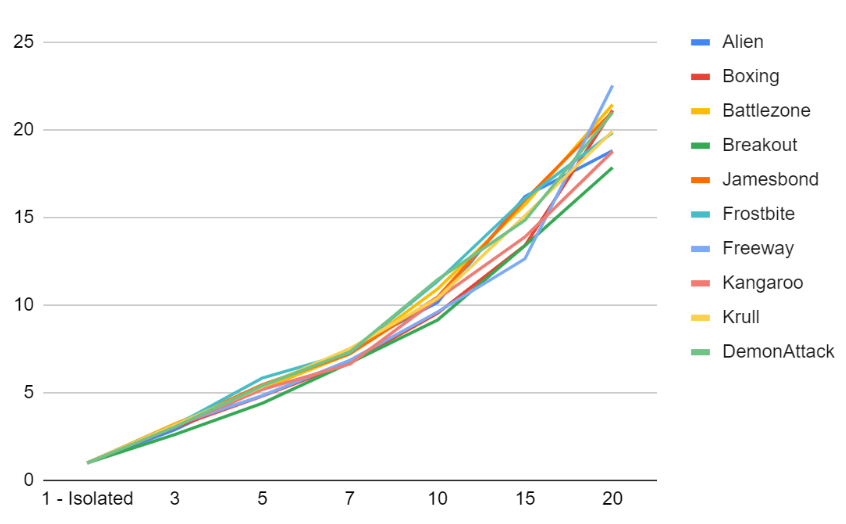}};
\node at (-3.3,0) [rotate=90,scale=0.9]{Normalized AR};
\node at (0,-2.2) [scale=0.9]{Number of agents, $N$};
\end{tikzpicture}
\caption{Normalized Agent Reward}
\label{fig:AR}
\end{subfigure}
\centering
\begin{subfigure}{2.6in}
\begin{tikzpicture}
\node at (0,0) {\includegraphics[scale=0.38]{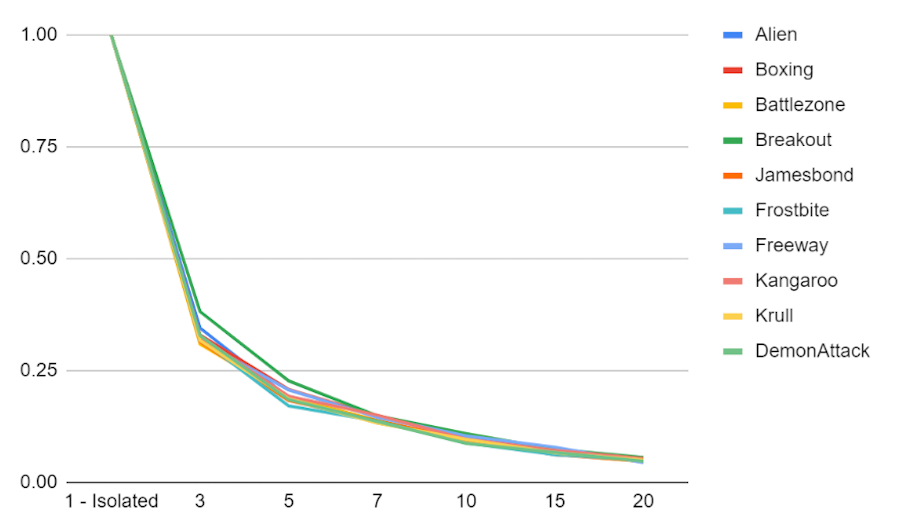}};
\node at (-3.3,0) [rotate=90,scale=0.9]{Normalized AT};
\node at (0,-2.2) [scale=0.9]{Number of agents, $N$};
\end{tikzpicture}
\caption{Normalized Agent Time}
\label{fig:AT}
\end{subfigure}
\caption{Normalized AR and AT.}
\label{fig:ARAt}
\end{figure*}


Figure \ref{fig:AR} shows the difference in normalized average reward result with different numbers of agents. In normalizing the reward, we take the reward in the random setting as 0 and the isolate agent as 1. We can see that across all games, the 
\distMT system achieves a linear increase across different numbers of agents in normalized Agent Reward. In general, the \distMT agent is able to achieve a $0.88N$ scaling on the reward. In this procedure: \begin{enumerate*}[label=\textit{\arabic*)}]
  \item We train the single agent (isolated Continual/Lifelong Learning) agent on a random permutation of 10 tasks.
  \item We repeat this experiment 5 times with 5 different random seeds. At the end of 10 tasks, we save the final/last model.
  \item We then use the 5 last saved models and evaluate the performance for each task for instance Alien.
  \item We finally report the results after evaluation on the best model (out of the 5 runs).
  \item We emulate the above experiment setting for \distMT experiments ( 3 agents, 5 agents, 7 agents, 10 agents, and 20 agents), and report the normalized AR.
\end{enumerate*}

We also demonstrate the decrease of Agent Time when the number of agents is rising. The agent time is measured by the percentage of number of frames for the agent on average to reach the reward in the Isolated agent that performs single agent lifelong learning. We use different numbers of agents that interpolate between 1 and 20 with the following specifications: \begin{enumerate*}[label=\textit{\arabic*)}]
  \item Hypothesis: increasing the number of agents will decrease the agent time to reach the desired reward on each game.
  \item Independent variable: Number of agents (1, 3, 5, 7, 10, 15, 20), Atari environments
  \item Dependent variable: Agent Time (AR) - Percentage of number of frames to reach the reward of the isolated agent
  \item Procedure: Randomly sample 10 unique tasks from the 20 total tasks.Run training on the \distMT system with different iterations per epoch (1, 3, 5, 7, 10, 15, 20) and train all 10 tasks.\end{enumerate*}

  \begin{figure*}[ht]
\centering
\begin{tikzpicture}
\node at (0,0) {\includegraphics[scale=0.45]{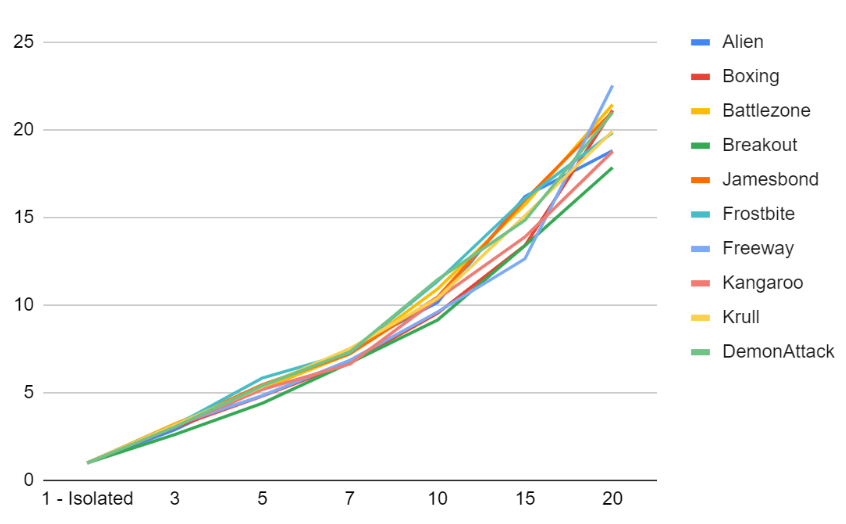}};
\node at (-3.8,0) [rotate=90,scale=0.9]{Number of frames};
\node at (0,-2.4) [scale=0.9]{Number of agents, $N$};
\end{tikzpicture}
\caption{Number of frames required by \distMT agents to reach x\% 	\{20, 40, 60, 80, 99\} of the Single Agent reward}
\label{fig:threshold}
\end{figure*}

Figure \ref{fig:AT} shows the difference in AT with different numbers of agents. We can see that across all games, our proposed \distMT system achieves a continuous decrease across different numbers of agents in Agent Time.


Figure \ref{fig:threshold} indicates the Agent Time (measured in terms of number of frames at the time of training). The frame rate is ~60 fps (frames/second). Here we show that scaling the number of agents leads to speedup of performance. On the y-axis, we have the number of frames. On the x-axis, we have the isolated agent (Lifelong Learning agent without experience sharing), 3 agents \distMT, 5 agents \distMT, 7 agents \distMT, 10 agents \distMT, and 20 agents \distMT. The following procedure has been adopted to achieve the above plot: \begin{enumerate*}[label=\textit{\arabic*)}]
\item We log the agent reward from the single agent (Lifelong Learning agent without experience sharing) experiment along with the number of frames. We then perform the \distMT experiment for 3 agents, 5 agents, 7 agents, 10 agents, 15 agents, and 20 agents.
 \item In each \distMT experiment, we measure the number of frames taken to achieve 20\% , 40\%, 60\%, 80\% and 99\% of the isolated agent reward, and we plot .
 \item This plot will help us guide new design experiments to balance the trade-off between utility and computational time.
\end{enumerate*}

\section{Related Work}

We consider the sample-complexity setup of distributed multi-task RL under the contextual MDP framework, where $N$ agents receive tasks specified by unknown contexts from a pool of $M$ tasks and they share information with each other through a server. 
Below, we contrast our work with related work in the literature.

\paragraph{Multi-task RL and Lifelong RL.} 
Multi-task RL studied in \citet{yang2020multi,hessel2019multi,brunskill2013sample,fifty2021efficiently,zhang2021provably,sodhani2021multi} assumes that tasks are chosen from a known finite set, and in \citet{yang2020multi,wilson2007multi,brunskill2013sample,sun2021temple}, tasks are sampled from a fixed distribution, and in both, task identities are assumed to be known. By contrast, our setting provides theoretical guarantees for task sequences whose identities are unknown. Another closely related line of work is lifelong RL, which studies how to learn to solve a streaming sequence of tasks.
Historically many works on lifelong RL \citep{ammar2014online,brunskill2014pac,abel2018state,abel2018policy,lecarpentier2021lipschitz} assume that the tasks are i.i.d. (similar to multi-task RL). 
There are works for adversarial sequences, but they all assume the tasks are known when assigned to the agent and most of them are purely empirical \citep{xie2021lifelong}. The work by \citet{isele2016using} uses contexts to enable zero-shot learning like here, but it (as well as most works above) does not provide formal regret or sample-complexity guarantees.

\paragraph{Distributed/Multi-agent RL.} Distributed/Multi-agent RL is a domain with a relatively long history, beginning from classical algorithms in the
fully-cooperative setting \citep{boutilier1996planning}, where all agents share identical reward functions to setting with multi-agent MDPs \citep{lauer2004reinforcement} and it has recently re-emerged due to advances in single-agent RL techniques. In the most closely related line of work, \citet{zhang2018networked, zhang2018fully, dubey2021provably} study a more general heterogeneous reward setting, where each agent has unique rewards. However, from a multi-task learning viewpoint, in all these works, each agent is assigned to only one fixed task, i.e., $M=N$, and its goal is to learn optimal policy for only that task. In contrast, in our work, $M$ and $N$ are not necessarily the same, and the goal is for all $N$ agents to achieve $\epsilon$-optimal policies for all $M$ tasks.

\paragraph{Contextual MDP and multi-objective RL.}
Our setup is closely related to the exploration problem studied in the contextual MDP literature. Most contextual MDP works allow adversarial contexts, but a majority of them focuses on the tabular setup \citep{abbasi2014online,hallak2015contextual,modi2018markov,modi2020no,levy2022learning,wu2021accommodating}, whereas our setup allows continuous state and action spaces. \citet{kakade2020information} and \citet{du2019continuous} allow continuous state and action spaces, but the former assumes a planning oracle with unclear computational complexity and the latter focuses on only LQG problems.

\section{Conclusion}
Motivated by DARPA's ShELL program, we conducted a comprehensive theoretical and empirical study on distributed multi-task RL. In this framework, $N$ agents collaborate to solve a set of $M$ tasks, without prior knowledge of the task identities. To address this challenge, we formulated the problem using linearly parameterized contextual MDPs, where each task is represented by a context that specifies its transition dynamics and rewards. To tackle this problem, we introduced a novel algorithm called \distMT that enables agents to first identify the tasks and then leverage a central server to facilitate information sharing, ultimately obtaining $\epsilon$-optimal policies for all $M$ tasks. Our theoretical analysis establishes that the total number of episodes required for a single agent to execute \distMT is bounded by $\Otilde({d^3H^6(\epsilon^{-2}+\csep^{-2})}\cdot M/N)$, where $\csep>0$ is a constant characterizing task separability, $H$ is horizon of each episode and $d$ is the feature dimension of the dynamics and rewards. Notably, \distMT significantly improves the sample complexity compared to non-distributed settings by a factor of $1/N$. In the non-distributed approach, each agent independently learns $\epsilon$-optimal policies for all $M$ tasks, necessitating $\Otilde(d^3H^6M\epsilon^{-2})$ episodes per agent. Our extensive numerical experiments using OpenAI Gym Atari environments provided empirical evidence supporting the efficacy and performance of our proposed methodology. That said, our work's limitations motivate further investigations in the following directions: 
\begin{enumerate*}[label=\textit{\arabic*)}]
    \item extension to more general class of MDPs, potentially using general function approximation/representation tools,
    \item studying settings with adversarial streaming sequence of tasks.
\end{enumerate*}


\newpage
\bibliographystyle{apalike}
\bibliography{main}

\newpage
\appendix

\section{Proofs of Section \ref{sec:alg}}\label{appx:proofofthm}
In order to prove Theorem \ref{thm:main}, and Lemmas \ref{lemm:confidence1} and \ref{lemm:confidence2}, we first state the following lemma.
\begin{lemma}\label{lemm:keylemma}
    Fix $\delta\in(0,1)$ and $\epsilon\in(0,1)$ and $\beta = c_{\beta} Hd\sqrt{\log( dHM\delta^{-1}\epsilon^{-1})}$ for some $c_{\beta}>0$, $K = c_{K}d^3H^6\log(dHM\delta^{-1}\epsilon^{-1})/\epsilon^2$ for some $c_{K}>0$. Let $\Dc = \{(s_h^k,a_h^k,r_h^k)\}_{(h,k)\in[H]\times[K]}$ be the dataset generated in Algorithm \ref{alg:explore} in an environment defined with MDP $\Mc=(\Sc,\Ac, H,\Pb, r)$ with unknown parameters $\mathbb P$ and $r$ and be the input to the Algorithm \ref{alg:planning}. Let $\Lambdab_h^k$, $u_h^k$, $\thetab_h^k$, $Q_h^k$, $V_h^k$ be defined as in Algorithm \ref{alg:explore} and $\Lambdab_h$, $u_h$, $\thetab_h$, $Q_h$, $V_h$ be defined as in Algorithm \ref{alg:planning} and $\pi_h(.) = \argmax_{a\in\Ac}Q_h(.,a)$. Then with probability at least $1-\delta$, it holds that
    \begin{align}\label{eq:first}
       V_1^\ast(s_0)\leq V_1(s_0)
    \end{align}
and
    \begin{align}\label{eq:second}
     V_1(s_0)-V_1^\pi(s_0)\leq \epsilon.
    \end{align}
\end{lemma}

\begin{proof}
.

{\bf First part:}

We define the backups as
\begin{align} \tilde\thetab_h\coloneqq\etab_h+\int_{\Sc}V_{h+1}(s^\prime)d\mub_h(s^\prime)\label{eq:tildethetahk},
\end{align}
Thanks to the linear MDP structure in Assumption \ref{assum:linearMDP}, it holds that 

\begin{align}\label{eq:PVh+1linearform}
    r_h(s,a)+\mathbb{P}_h\left[V_{h+1}(.)\right](s,a) = \left\langle\tilde\thetab_{h},\phib(s,a)\right\rangle.
\end{align}

\begin{align}
    \tilde\thetab_h - \thetab_h&= \tilde\thetab_h-\Lambdab_h^{-1}\sum_{\tau=1}^{K}\phib_h^\tau [r_h^\tau+V_{h+1}(s_{h+1}^\tau)]\nn\\
    &= \Lambdab_h^{-1}\left(\Lambdab_h\tilde\thetab_h-\sum_{\tau=1}^{K}\phib_h^\tau [r_h^\tau+V_{h+1}(s_{h+1}^\tau)]\right)\nn\\
&=\underbrace{\Lambdab_h^{-1}\tilde\thetab_h}_{\qb_1}\underbrace{-\Lambdab_h^{-1}\left(\sum_{\tau=1}^{K}\phib_h^\tau\left( V_{h+1}(s_{h+1}^\tau)-\mathbb{P}_h[V_{h+1}(.)](s_h^\tau,a_h^\tau)\right)\right)}_{\qb_2}\nn.
\end{align}

Thus, in order to upper bound $\norm{\tilde\thetab_h-\thetab_h}_{\Lambdab_h}$, we bound $\norm{\qb_1}_{\Lambdab_h}$ and $\norm{\qb_2}_{\Lambdab_h}$ separately.

From Lemma Assumption \ref{assum:linearMDP}, we have
\begin{align}\label{eq:q1phi1}
\norm{\qb_1}_{\Lambdab_h}=\norm{\tilde\thetab_h}_{\Lambdab_h^{-1}}\leq \norm{\tilde\thetab_h}_2\leq H\sqrt{d}.
\end{align}

Thanks to Lemma \ref{lemm:lemmaD.4inJinetal}, for all $h\in[H]$, with probability at least $1-\delta$, it holds that
\begin{align} \label{eq:q2phi1}
    \norm{\qb_2}_{\Lambdab_h}\leq \norm{\sum_{\tau=1}^{K}\phib_h^\tau\left( V_{h+1}(s_{h+1}^\tau)-\mathbb{P}_h[V_{h+1}(.)](s_h^\tau,a_h^\tau)\right)}_{\Lambdab_h^{-1}}\leq c_0Hd\sqrt{\log((c_\beta+1) dHK/\delta)},
\end{align}
where $c_0$ and $c_\beta$ are two independent absolute constants. Combining \eqref{eq:q1phi1} and \eqref{eq:q2phi1}, for all $h\in[H]$, with probability at least $1-\delta$, it holds that

\begin{align}
    \norm{\thetab_h-\tilde\thetab_h}_{\Lambdab_h}\leq cHd\sqrt{\log( dHK/\delta)}=\beta\nn
\end{align}
for some absolute constant $c>0$. Therefore, for all $h\in[H]$, with probability at least $1-\delta$, it holds that
\begin{align}\label{eq:confidence1}
    \abs{r_h(s,a)+\mathbb{P}_h\left[V_{h+1}(.)\right](s,a) - \langle \thetab_h,\phib(s,a)\rangle} &= \abs{\langle \tilde\thetab_h-\thetab_h,\phib(s,a)\rangle}\nn\\
    &\leq \beta\norm{\phib(s,a)}_{\Lambdab_h^{-1}}.
\end{align}

Note that conditioned on the event defined in \eqref{eq:confidence1} for all $(s,a)\in\Sc\times\Ac$, it holds that
\begin{align}\label{eq:confidence2}
&\abs{\left\langle\thetab_h,\phib(s,a)\right\rangle-Q_h^\pi(s,a)-\mathbb{P}_h\left[V_{h+1}(.)-V_{h+1}^\pi(.)\right](s,a)}\nn\\
&=\abs{\left\langle\thetab_h,\phib(s,a)\right\rangle-r_h(s,a)-\mathbb{P}_h\left[V_{h+1}(.)\right](s,a)}\nn\\
&\leq\beta\norm{\phib(s,a)}_{\left(\Lambdab_h^k\right)^{-1}},
\end{align}
for any policy $\pi$. Now, we are ready to prove the first part of the lemma by induction. The statement holds for $H$ because $Q_{H+1}(.,.)=Q_{H+1}^\ast(.,.)=0$ and thus conditioned on the event defined in \eqref{eq:confidence2}, for all $(s,a)\in\Sc\times\Ac$, we have

\begin{align}
   \abs{\left\langle\thetab_H,\phib(s,a)\right\rangle-Q_H^{\ast}(s,a)}\leq \beta\norm{\phib(s,a)}_{\Lambdab_H^{-1}}.\nn
\end{align}
Therefore, for all $(s,a)\in\Sc\times\Ac$, we have
\begin{align}
    Q^\ast_H(s,a)\leq \left\langle\thetab_H,\phib(s,a)\right\rangle+\beta\norm{\phib(s,a)}_{\Lambdab_H^{-1}}\nn.
\end{align}
Since $Q_H^\ast(s,a)\leq H$, we have
\begin{align}
    Q^\ast_H(s,a)\leq \min\left\{\left\langle\thetab_H,\phib(s,a)\right\rangle+\min\left\{\beta\norm{\phib(s,a)}_{\Lambdab_H^{-1}},H\right\},H\right\} = Q_H(s,a)\nn,
\end{align}
which implies that
\begin{align}
    V^\ast_H(s) = \max_{a\in\Ac}Q_H^\ast(s,a)\leq \max_{a\in\Ac}Q_H^(s,a) = V_H(s).\nn
\end{align}
Now, suppose the statement holds at time-step $h+1$ and consider time-step $h$. Conditioned on the event defined in \eqref{eq:confidence2}, for all $(s,a)\in\Sc\times\Ac$, we have
\begin{align}
    0&\leq \left\langle\thetab_h,\phib(s,a)\right\rangle-Q_h^{\ast}(s,a)-\mathbb{P}_h\left[V_{h+1}(.)-V_{h+1}^{\ast}(.)\right](s,a)+\beta\norm{\phib(s,a)}_{\Lambdab_h^{-1}}\nn\\
    &\left\langle\thetab_h,\phib(s,a)\right\rangle-Q_h^{\ast}(s,a)+\beta\norm{\phib(s,a)}_{\Lambdab_h^{-1}}.\tag{Induction assumption}
\end{align}
Since $Q_h^\ast(s,a)\leq H$, we have
\begin{align}
    Q^\ast_h(s,a)\leq \min\left\{\left\langle\thetab_h,\phib(s,a)\right\rangle+\min\left\{\beta\norm{\phib(s,a)}_{\Lambdab_h^{-1}},H\right\},H\right\} = Q_h(s,a)\nn,
\end{align}
which means that

\begin{align}
    V^\ast_h(s) = \max_{a\in\Ac}Q_h^\ast(s,a)\leq \max_{a\in\Ac}Q_h^(s,a) = V_h(s).\nn
\end{align}
This completes the proof.

{\bf Second part:}

From Lemma 3.2 in \cite{wang2020reward}, with probability at least $1-\delta$, it holds that
\begin{align}\label{eq:wang}
    V_1^\ast(s_0,u/H)\leq c^\prime\sqrt{d^3H^4\log(dKH/\delta)/K},
\end{align}
for some absolute constant $c^\prime>0$.

Note that
\begin{align}
    V_1(s_0) - V_1^\pi(s_0) &=Q_1(s_0,\pi_1(s_0)) - Q_1^\pi(s_0,\pi_1(s_0))\nn\\
    &\leq \mathbb E_{s_2\sim\mathbb P_1(.|s_0,\pi_1(s_0))}\left[V_{2}(s_2)+2u_1(s_0,\pi_1(s_0))-V_2^\pi(s_2)\right]\nn\\
    &\leq E_{s_3\sim\mathbb P_2(.|s_1,\pi_2(s_1))}E_{s_2\sim\mathbb P_1(.|s_0,\pi_1(s_0))}\left[V_{3}(s_3)-V_3^\pi(s_3)+2u_1(s_0,\pi_1(s_0))++2u_2(s_2,\pi_2(s_2))\right]\nn\\
    \vdots\nn\\
    &\leq V_1^\pi(s_0,u)\nn\\
    &\leq V_1^\ast(s_0,u)\nn\\
    &=H V_1^\ast(s_0,u/H)\nn\\
    &\leq c^\prime\sqrt{d^3H^6\log(dKH/\delta)/K} \tag{Eqn \eqref{eq:wang}}.
\end{align}
Therefore, by taking  $K = c_{K}d^3H^6\log(dHM\delta^{-1}\epsilon^{-1})/\epsilon^2$ for some $c_{K}>0$, we have
\begin{align}
     V_1(s_0)-V_1^\pi(s_0)\leq c^\prime\sqrt{d^3H^6\log(dKH/\delta)/K} \leq \epsilon,
\end{align}
which completes the proof.
\end{proof}


\subsection{Proof of Lemma \ref{lemm:confidence1}}
Both inequalities in Lemma \ref{lemm:confidence1} can be proven using Lemma \ref{lemm:keylemma}. Note that for every $(i,t)\in[N]\times[T]$, $V_{1,i}^{1,t}$ is computed based on output quantities of Algorithm \ref{alg:planning} whose input is a dataset generated by interacting with task $m_{i,t}$ environment (see Lines \ref{line:K1explore} and \ref{line:K1plan} in Algorithm \ref{alg:main}). Thus, from \eqref{eq:first}, with probability at least $1-\delta$, it holds that
\begin{align}
    0\leq V_{1,i}^{1,t}(s_0)-V_{m_{i,t},1}^\ast(s_0).
\end{align}

Now, let $\tilde\pi_i^t = \{\tilde\pi_{h,i}^t\}_{h=1}^H$, where $\tilde\pi_{h,i}^t(.) = \argmax_{a\in\Ac}Q_h(.,a)$, $Q_h(.,.) =  \min\left\{\langle\thetab_{h,i}^{1,t},\phib(.,.)\rangle+u_h(.,.),H\right\}$,  $u_h(.,.) = \min\left\{\beta_1\norm{\phib(.,.)}_{\left(\Lambdab_{h,i}^{1,t}\right)^{-1}},H\right\}$. Therefore, from \eqref{eq:second}, with probability at least $1-\delta$, it holds that
\begin{align}
   V_{1,i}^{1,t}(s_0)-V_{m_{i,t},1}^\ast(s_0)\leq  V_{1,i}^{1,t}(s_0)-V_{m_{i,t},1}^{\tilde\pi_i^t}(s_0)\leq\csep/8,
\end{align}
which completes the proof.

\subsection{Proof of Lemma \ref{lemm:separ}}
First, we prove that if $\abs{V_{1,i}^{1,t}(s_0)-V_{1,j}^{1,t^\prime}(s_0)}>\csep/2$, then $m_{i,t}$ and $m_{j,t^\prime}$ are two different tasks. We have
\begin{align}
   \abs{V_{1,m_{i,t}}^\ast(s_0)-V_{1,m_{j,t^\prime}}^\ast(s_0)}&=\abs{V_{1,m_{i,t}}^\ast(s_0)-V_{1,i}^{1,t}(s_0)+V_{1,i}^{1,t}(s_0)-V_{1,j}^{1,t^\prime}(s_0)+V_{1,j}^{1,t^\prime}(s_0)-V_{1,m_{j,t^\prime}}^\ast(s_0)}\nn\\
   &\geq \abs{V_{1,i}^{1,t}(s_0)-V_{1,j}^{1,t^\prime}(s_0)}-\abs{V_{1,m_{i,t}}^\ast(s_0)-V_{1,i}^{1,t}(s_0)+}-\abs{V_{1,j}^{1,t^\prime}(s_0)-V_{1,m_{j,t^\prime}}^\ast(s_0)}\tag{Triangle Inequality}\\
   &>\csep/2-\csep/8-\csep/8 = \csep/4\tag{Lemma \ref{lemm:confidence1}}
\end{align}
which means that $V_{1,m_{i,t}}^\ast(s_0)\neq V_{1,m_{j,t^\prime}}^\ast(s_0)$, and therefore, $m_{i,t}$ and $m_{j,t^\prime}$ cannot be the same tasks.

Now, we prove that $\abs{V_{1,i}^{1,t}(s_0)-V_{1,j}^{1,t^\prime}(s_0)}\leq\csep/2$, then $m_{i,t}$ and $m_{j,t^\prime}$ are the same tasks. We have
\begin{align}
   \abs{V_{1,m_{i,t}}^\ast(s_0)-V_{1,m_{j,t^\prime}}^\ast(s_0)}&=\abs{V_{1,m_{i,t}}^\ast(s_0)-V_{1,i}^{1,t}(s_0)+V_{1,i}^{1,t}(s_0)-V_{1,j}^{1,t^\prime}(s_0)+V_{1,j}^{1,t^\prime}(s_0)-V_{1,m_{j,t^\prime}}^\ast(s_0)}\nn\\
   &\leq \abs{V_{1,i}^{1,t}(s_0)-V_{1,j}^{1,t^\prime}(s_0)}+\abs{V_{1,m_{i,t}}^\ast(s_0)-V_{1,i}^{1,t}(s_0)+}+\abs{V_{1,j}^{1,t^\prime}(s_0)-V_{1,m_{j,t^\prime}}^\ast(s_0)}\tag{Triangle Inequality}\\
   &\leq\csep/2+\csep/8+\csep/8 = 3\csep/4\tag{Lemma \ref{lemm:confidence1}},
\end{align}
which means that $m_{i,t}$ and $m_{j,t^\prime}$ cannot be two different tasks as $ \abs{V_{1,m_{i,t}}^\ast(s_0)-V_{1,m_{j,t^\prime}}^\ast(s_0)}$ is not greater than $\csep$ (See Assumption \ref{assum:separ}).


\subsection{Proof of Lemma \ref{lemm:confidence2}}
Thanks to the definition of optimal policy and $V_{m_{i,t},1}^\ast(s_0)$, it is trivial to show the first inequality holds and $V_{m_{i,t},1}^{\pi_i^t}(s_0)\leq V_{m_{i,t},1}^\ast(s_0)$. To prove the second inequality, we use Lemma \ref{lemm:keylemma}. Note that for every $(i,t)\in[N]\times[T]$, $V_{1,i}^{1,t}$ is computed based on output quantities of Algorithm \ref{alg:planning} whose input is a dataset generated by interacting with task $m_{i,t}$ environment (see Lines \ref{line:K2explore} and \ref{line:K2plan} in Algorithm \ref{alg:main}). Therefore, for all $(i,t)\in[N]\times[T]$, with probability at least $1-\delta$, it holds that
\begin{align}
   V_{m_{i,t},1}^\ast(s_0) - V_{m_{i,t},1}^{\pi_i^t}(s_0)&\leq V_{1,i}^{2,t}(s_0) - V_{m_{i,t},1}^{\pi_i^t}(s_0)\tag{Eqn. \eqref{eq:first}}\\
   &\leq \epsilon \tag{Eqn. \eqref{eq:second}},
\end{align}
which completes the proof.

\subsection{Proof of Lemma \ref{lemm:T}}
    Let $\mathbb{I}_{i,t}^m$ be an indicator random variable, which is if task $m$ is assigned to agent $i$ at round $t$ and 0 otherwise. Let $k_m = \sum_{i\in[N]}\sum_{t\in[T]}\mathbb{I}_{i,t}^m$ be the random variable specifying the number of times task $m$ were assigned to an agent over the course of $T$ rounds. Therefore, we have
    \begin{align}
        \mu_m = \mathbb{E}[k_m] = \frac{NT}{M}.
    \end{align}

Using multiplicative Chernoff bound, we have
\begin{align}
    \mathbb{P}(k_m<1)\leq e^{\frac{-\left(1-\frac{1}{\mu_m}\right)^2\mu_m}{2}}.
\end{align}
Our choice of $T$ guarantees that
\begin{align}
    \mathbb{P}(\exists m\in[M], k_m<1)\leq \delta,
\end{align}
which completes the proof.


\section{Auxiliary lemmas}\label{sec:auxiliary}
\paragraph{Notations.}

$\Nc_\epsilon(\Vc)$ denotes the $\epsilon$-covering number of the class $\Vc$ of functions mapping $
\Sc$ to $\mathbb{R}$ with respect to the distance ${\rm dist}(V,V^\prime)=\sup_s\abs{V(s)-V^\prime(s)}$.

\begin{lemma}[Lemma D.4 in \citet{jin2020provably}]\label{lemm:lemmaD.4inJinetal}
Let $\{s_\tau\}_{\tau=1}^\infty$ be a stochastic process on state space $\Sc$ with corresponding filtration $\{\Fc_\tau\}_{\tau=0}^\infty$. Let $\{\phib_\tau\}_{\tau=0}^\infty$ be an $\mathbb{R}^d$-valued stochastic process where $\phib_\tau\in\Fc_{\tau-1}$, and $\norm{\phib_\tau}\leq 1$. Let $\Lambdab_k= \Iden_{d}+\sum_{\tau=1}^{k-1}\phib_\tau\phib_\tau^\top$. Then with probability at least $1-\delta$, for all $k\geq 0$ and $V\in\Vc$ such that $\sup_{s\in\Sc}\abs{V(s)}\leq H$, we have
\begin{align}
    \norm{\sum_{\tau=1}^k\phib_\tau.\left(V(s_\tau)-\mathbb{E}\left[V(s_\tau)\vert\Fc_{\tau-1}\right]\right)}_{\Lambdab_k^{-1}}^2\leq 4H^2\left(\frac{{d}}{2}\log\left(\frac{k+\la}{\la}\right)+\log\left(\frac{\Nc_\epsilon(\Vc)}{\delta}\right)\right)+\frac{8k^2\epsilon^2}{\la}.\nn
\end{align}
\end{lemma}
\begin{lemma}\label{lemm:coveringnumber} For any $\epsilon>0$, the $\epsilon$-covering number of the Euclidean
ball in $\mathbb{R}^{d}$ with radius $R>0$ is upper bounded by $(1+2R/\epsilon)^{d}$.
\end{lemma}

\begin{lemma}\label{lemm:coveringnumberQ1}
For a fixed $w$, let $\Vc$ denote a class of functions mapping from $\Sc$ to $\mathbb{R}$ with following parametric form 
\begin{align}
     V(.) = \min\left\{\max_{a\in\Ac}\left\langle\y,\phib(.,a)\right\rangle+\beta\sqrt{\phib(.,a)^\top \Yb\phib(.,a)},H\right\}\nn,
\end{align}
where the parameters $\beta\in\mathbb{R}$, $\y\in\mathbb{R}^{d}$, and $\Yb\in\mathbb{R}^{d\times d}$ satisfy $0\leq\beta\leq B$, $\norm{\y}\leq y$, and $\norm{\Yb}\leq \la^{-1}$. Assume $\norm{\phib(s,a)}\leq 1$ for all $(s,a)\in\Sc\times\Ac$. Then 
\begin{align}
   \log\left(\Nc_\epsilon(\Vc)\right)\leq d\log(1+4y/\epsilon)+d^2\log\left(\frac{1+8B^2\sqrt{d}}{\la\epsilon^2}\right).\nn
\end{align}
\end{lemma}

\begin{proof}
First, we reparametrize $\Vc$ by letting $\tilde\Yb = \beta^2\Yb$. We have
\begin{align}
     V(.) = \min\left\{\max_{a\in\Ac}\left\langle\y,\phib(.,a)\right\rangle+\sqrt{\phib(.,a)^\top \tilde\Yb\phib(.,a)},H\right\}\nn,
\end{align}
for $\norm{\y}\leq y$ and $\norm{\tilde\Yb}\leq \frac{B^2}{\la}$. For any two functions $V_1,V_2\in\Vc$ with parameters $\left(\y^1,\tilde\Yb^1\right)$ and $\left(\y^2,\tilde\Yb^2\right)$, respectively, we have
\begin{align}
  {\rm dist}(V_1,V_2)&\leq \sup_{(s,a)\in\Sc\times\Ac}\left|\left[\left\langle\y^1,\phib(s,a)\right\rangle+\sqrt{\phib(s,a)^\top \tilde\Yb^1\phib(s,a)}\right]\right.\nn\\
  &\left.-\left[\left\langle\y^2,\phib(s,a)\right\rangle+\sqrt{\phib(s,a)^\top \tilde\Yb^2\phib(s,a)}\right]\right|\nn\\
  &\leq\sup_{\phib:\norm{\phib}\leq 1}\abs{\left[\left\langle\y^1,\phib\right\rangle+\sqrt{\phib^\top \tilde\Yb^1 \phib}\right]-\left[\left\langle\y^2,\phib\right\rangle+\sqrt{\phib^\top \tilde\Yb^2 \phib}\right]}\nn\\
  &\leq\sup_{\phib:\norm{\phib}\leq 1}\abs{\left\langle\y^1-\y^2,\phib\right\rangle}+\sup_{\phib:\norm{\phib}\leq 1}\sqrt{\abs{\phib^\top \left(\tilde\Yb^1-\tilde\Yb^2\right) \phib}}\tag{because $\abs{\sqrt{a}-\sqrt{b}}\leq\sqrt{\abs{a-b}}$ for $a,b\geq0$}\\
  &= \sqrt{\norm{\tilde\Yb^1-\tilde\Yb^2}}\nn\\
  &\leq\norm{\y^1-\y^2}+\sqrt{\norm{\tilde\Yb^1-\tilde\Yb^2}_F}\label{eq:lastinthecoveringproofQ1}.
\end{align}
Let  $\Cc_{\y}$ be $\epsilon/2$-covers of $\{\y\in\mathbb{R}^{d}:\norm{\y}\leq y\}$ with respect to the $2$-norm and $\Cc_{\Yb}$ be an $\epsilon^2/4$-cover of $\{\Yb\in\mathbb{R}^{d\times d}:\norm{\Yb}_F\leq \frac{B^2\sqrt{d}}{\la}\}$, with respect to the Frobenius norm. By Lemma \ref{lemm:coveringnumber}, we know
\begin{align}
     \abs{\Cc_{\y}}\leq (1+4y/\epsilon)^{d},\quad\abs{\Cc_{\Yb}}\leq \left(\frac{1+8B^2\sqrt{d}}{\la\epsilon^2}\right)^{d^2}.\nn
\end{align}
According to \eqref{eq:lastinthecoveringproofQ1}, it holds that $\Nc_\epsilon(\Vc)\leq\abs{\Cc_{\y}}\abs{\Cc_{\Yb}}$, and therefore
\begin{align}
    \log\left(\Nc_\epsilon(\Vc)\right)\leq d\log(1+4y/\epsilon)+d^2\log\left(\frac{1+8B^2\sqrt{d}}{\la\epsilon^2}\right).\nn
\end{align}
\end{proof}

\end{document}